\documentclass{article}

\RequirePackage{xcolor}
\usepackage{microtype}
\usepackage{graphicx}
\usepackage{subfigure}
\usepackage{booktabs}
\usepackage{hyperref}

\usepackage[accepted]{icml2019}
\usepackage{mathtools}
\usepackage{amsmath}
\usepackage{amssymb}
\usepackage{amsthm}
\usepackage{hyperref}
\usepackage{acronym}
\usepackage{cleveref}
\usepackage[inline]{enumitem}
\usepackage[nodisplayskipstretch]{setspace}
\usepackage{bm}
\usepackage{todonotes}
\usepackage{siunitx}
\usepackage{booktabs}
\usepackage[compact]{titlesec}

\usepackage{natbib}

\setcitestyle{numbers,square,comma,sort&compress}
\acrodef{SVM}{Support Vector Machine}
\acrodef{GP}{Gaussian Process}
\acrodefplural{GP}{Gaussian Processes}
\acrodef{GPR}{\ac{GP} Regression}
\acrodef{GPC}{\ac{GP} Classification}
\acrodef{MAP}{Maximum a-Posteriori}
\acrodef{SVD}{Singular Value Decomposition}
\acrodef{OLS}{Ordinary Least Squares}
\acrodef{BRR}{Bayesian Ridge Regression}
\acrodef{RBF}{Radial Basis Function}
\acrodef{LR}{Logistic Regression}
\acrodef{BLR}{Bayesian \ac{LR}}
\acrodef{CDF}{cumulative distribution function}
\acrodef{PDF}{probability distribution function}
\acrodef{NN}{Neural Network}
\acrodef{BNN}{Bayesian \ac{NN}}
\acrodef{RBF}{Radial Basis Function}
\acrodef{RKHS}{reproducing kernel Hilbert space}
\acrodef{VI}{Variational Inference}
\acrodef{SVI}{Stochastic \ac{VI}}
\acrodef{EP}{Expectation Propagation}
\acrodef{MCMC}{Markov Chain Monte-Carlo}
\acrodef{PAV}{Pool Adjacent Violators}
\acrodef{NLL}{negative log likelihood}
\acrodef{MSE}{mean squared error}
\acrodef{PBL}{pinball loss}
\acrodef{ELBO}{evidence lower bound}

\newtheorem{definition}{Definition}
\newtheorem{theorem}{Theorem}

\theoremstyle{definition}
\newtheorem{remark}{Remark}


\renewcommand{\Pr}{{\mathsf{P}}}
\newcommand{\pr}{{\mathsf{p}}}
\newcommand{\gp}{{\mathsf{gp}}}

\newcommand{\f}{{\mathsf{f}}}

\newcommand{\B}{{\mathbf{B}}}
\newcommand{\C}{{\mathbf{C}}}

\newcommand{\K}{{\mathbf{K}}}

\renewcommand{\sb}{{\mathbf{s}}}

\newcommand{\wb}{{\mathbf{w}}}

\newcommand{\yb}{{\mathbf{y}}}

\newcommand{\zeros}{{\mathbf{0}}}

\newcommand{\posthoc}{\emph{post-hoc}}
\newcommand{\Posthoc}{\emph{Post-hoc}}

\newcommand{\ie}{\emph{i.e.}}
\newcommand{\eg}{\emph{e.g.}}

\newcommand{\GPB}{\textsc{GP-Beta}}

\Crefname{equation}{Eq.}{Eqs.}% {<type>}{<singular>}{<plural>}
\Crefname{figure}{Fig.}{Figs.}% {<type>}{<singular>}{<plural>}
\Crefname{section}{Sec.}{Secs.}% {<type>}{<singular>}{<plural>}
\Crefname{algorithm}{Alg.}{Algs.}% {<type>}{<singular>}{<plural>}
\Crefname{theorem}{Thm.}{Thms.}% {<type>}{<singular>}{<plural>}
\Crefname{definition}{Def.}{Defs.}% {<type>}{<singular>}{<plural>}
\Crefname{appendix}{Appx.}{Appxs.}
\Crefname{table}{Tab.}{Tabs.}
\icmltitlerunning{Distribution Calibration for Regression}

\begin{document}

\twocolumn[
\icmltitle{Distribution Calibration for Regression}

\icmlsetsymbol{equal}{*}

\begin{icmlauthorlist}
\icmlauthor{Hao Song}{uob}
\icmlauthor{Tom Diethe}{amazon}
\icmlauthor{Meelis Kull}{uot}
\icmlauthor{Peter Flach}{uob,ati}
\end{icmlauthorlist}

\icmlaffiliation{uob}{University of Bristol, Bristol, United Kingdom}
\icmlaffiliation{uot}{University of Tartu, Tartu, Estonia}
\icmlaffiliation{amazon}{Amazon Research, Cambridge, United Kingdom}
\icmlaffiliation{ati}{The Alan Turing Institute, London, United Kingdom}

\icmlcorrespondingauthor{Hao Song}{hao.song@bristol.ac.uk}

\icmlkeywords{Calibration, Regression, Gaussian Process}

\vskip 0.3in
]

\printAffiliationsAndNotice

\begin{abstract}
We are concerned with obtaining well-calibrated output distributions from regression models. Such distributions allow us to quantify the uncertainty that the model has regarding the predicted target value. 
We introduce the novel concept of distribution calibration, and demonstrate its advantages over the existing definition of quantile calibration.
We further propose a \posthoc{} approach to improving the predictions from previously trained regression models, using multi-output Gaussian Processes with a novel Beta link function.
The proposed method is experimentally verified on a set of common regression models and shows improvements for both distribution-level and quantile-level calibration.
\end{abstract}

\section{Introduction}
\label{sec:introduction}

% \HS{Reduce text}

% We are doing calibration regression.
With recent progress in predictive machine learning, many models are now capable of providing outstanding performance with respect to certain metrics, such as accuracy in classification or mean squared error in regression.
While such models are suitable for some tasks, they often cannot provide well-quantified uncertainties on the target variables. % given each model output.
In this paper we focus on this problem in the regression setting, extending concepts from the well-established framework of probability calibration for classification.

% What is calibration, What is calibrated probabilities.
In a classification task, a probabilistic prediction $s \in [0,1]$ for the positive class is calibrated if the following condition holds: among all the instances receiving this same prediction value $s$, the probability of observing a positive label is $s$.
%From a frequentist point of view, a $0.5$ estimated probability of rain is calibrated if, among all the days receiving this probability estimate, half of those days it indeed rained.
% A Bayesian might say that this $0.5$ is calibrated if it yields a conditional distribution over days about which we should be maximally uncertain whether it will rain or not.
%
% Calibration is good.
Having such calibrated outputs is important as they can be interpreted as degrees of uncertainty on the class, hence enabling quantitative approaches towards decision making, such as cost-sensitive classification \cite{Zadrozny2002}.
% But lots of models are not calibrated.
However, from simple models (\eg{} Na{\"i}ve Bayes) to complex ones (\eg{} (deep) \acp{NN}), poor calibration is often observed, irrespective of the model's complexity or its probabilistic nature \cite{Guo2017, Kull2017}. 
To mitigate this, several techniques have been proposed to apply \posthoc{} corrections to the outputs from trained classifiers such as Platt scaling \cite{Platt1999}, Isotonic regression \cite{Zadrozny2002}, and Beta calibration \cite{Kull2017}.
% reference to existing approach
%One of the well-known parametric approaches is Logistic calibration \cite{Platt1999}, which uses the logistic function to map unbounded classifier scores, such as those produced by the \ac{SVM}, into the interval $[0, 1]$ so as to represent better calibrated probabilities.
%Recently \cite{Kull2017} proposed the Beta calibration approach that allows more flexible adjustment beyond a simple sigmoid shape. 

% What is calibration in regression.
In the setting of regression, calibration has been traditionally defined through predicted credible intervals \cite{Gneiting2007probabilistic,Fasiolo2017,Kuleshov18}, where a $0.95$ predicted credible level (e.g. conditional quantile) is calibrated if marginally $95\%$ of the true target values are below or equal to it. 
Having a quantile-calibrated regressor is particularly useful for certain forecasting tasks such as energy usage \cite{Gneiting2007probabilistic} and supply chain optimisation \cite{jatta2016empirical}. 
% \MK{here the term 'quantile' jumped in without relating to previous sentences}
% Quantile regression and condition densities estimation.
%In terms of predictive models that are able to achieve out-of-the-box calibrations, existing approaches can be loosely divided into two categories.
Existing approaches which aim to obtain calibrated predictions as part of the training, can be loosely divided into two categories: 
(i) quantile regression \cite{Koenker2001}, where it has shown that generalised additive models can be applied to yield better-calibrated quantiles \cite{Fasiolo2017}; (ii) direct application of conditional density estimators \cite{Bishop2006,Sugiyama2010,orallo2014probabilistic} to obtain an estimated \ac{CDF}, which can then be used to generate corresponding credible intervals.

% Post-hoc is also worth considering
While such approaches can be good options when simple models will suffice, they are less suitable when employing (possibly extant) specialised models, such as (pre-trained) deep \ac{NN} models.
% Recent reports with the isotonic approach
Unlike the case of classification, the \posthoc{} calibration approach in regression has been left largely unexplored. %, since practitioners have been primarily interested in the estimated mean value and not many models are able to provide probabilistic outputs other than a (potentially biased) standard deviation.
Recently, \cite{Kuleshov18} proposed a \posthoc{} approach that applies isotonic regression to match the predicted \ac{CDF} and empirical frequency, so that the final results are better calibrated in the quantile sense.
% While the proposed approach shows improvements on prediction error and quantile error, the estimated \ac{PDF} can be highly fluctuating due to the step-wise nature of isotonic regression. 

% Quantile-level calibration is not enough
While quantile-level calibration is useful in certain scenarios, it is defined uniformly over the entire input space and does not ensure calibration for a particular prediction, unlike the classification setting.
For instance, a quantile-calibrated regressor cannot always guarantee that all instances receiving an estimated mean of $\mu$ and standard deviation $\sigma$ are indeed distributed as a Gaussian distribution with the same moments.
% Our work and contribution
In this paper, we focus on a \posthoc{} method that aims to achieve distribution-level calibration and hence gives more accurate uncertainty information for a continuous target variable.
%We introduce the concept of distribution calibration, and demonstrate that being calibrated on a distribution-level will naturally lead to calibrated quantiles.
%Measures to examine and verify such distribution calibrations are also given and discussed.
%Finally, we propose a \ac{GP}-based approach to solve the task of \posthoc{} density calibration.
%The proposed approach models the distribution over calibration parameters, and uses a novel differential link function to calibrate any given regression outputs.
% To ensure the scalability of the model, we further provide a solution based on stochastic variational inference together with induced pseudo-points.
% Finally, the proposed approach is experimentally analysed on different regression models such as \ac{GP} regression \cite{Rasmussen2006} and \ac{NN} regression with uncertainty approximated by dropouts \cite{Gal2016}.

\begin{figure*}[t]
    \centering
    \includegraphics[width=0.63\textwidth]{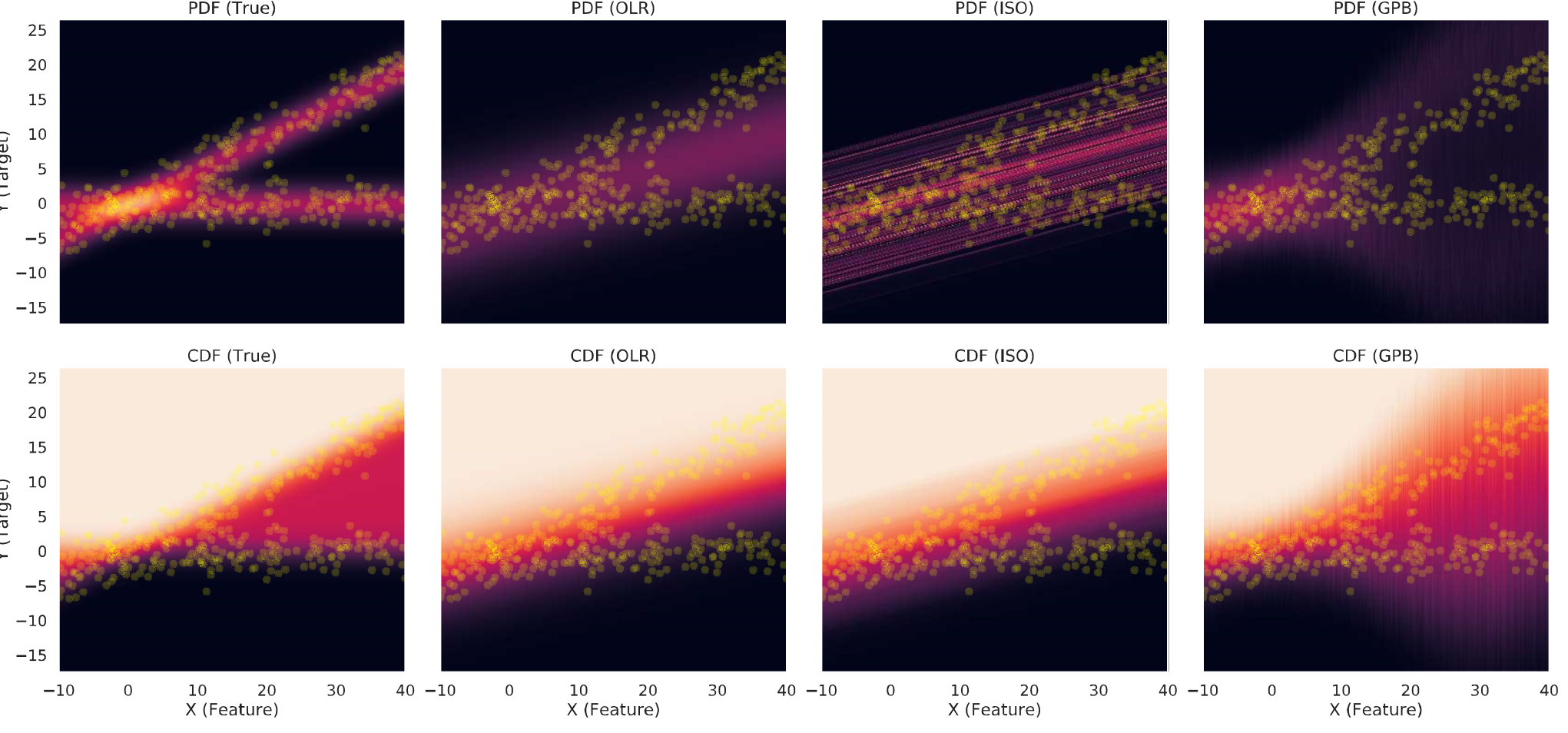}
    \caption{Applying quantile and distribution calibration on a synthetic dataset. The column on the left shows the true conditional PDF / CDF of the dataset, with the yellow points being the observed data. An ordinary linear regression is fitted to the data and the predictions are shown in the second column.
    An isotonic regression and a \GPB{} model are trained to calibrate the \ac{OLS} outputs on quantile and distribution level respectively, giving the right two columns.}
    \label{fig:toy_res}
\end{figure*}

\begin{figure*}[t]
    \centering
    \includegraphics[width=0.45\textwidth]{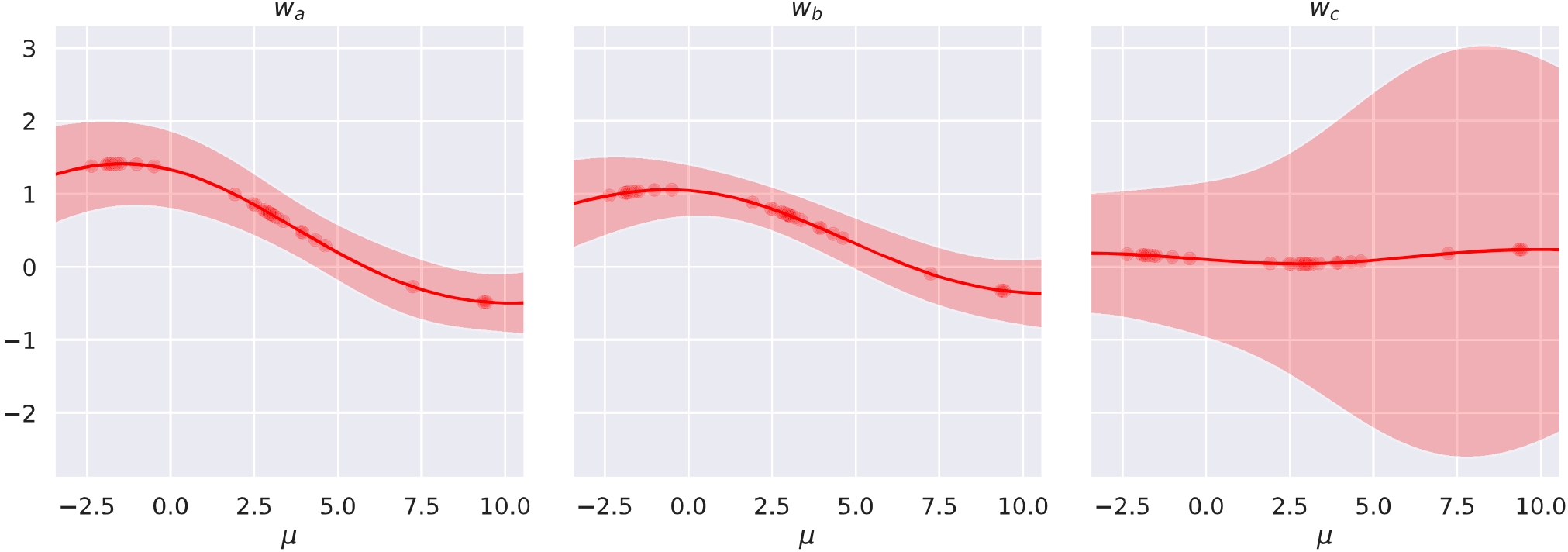}\hspace{2cm}
    \includegraphics[width=0.25\textwidth]{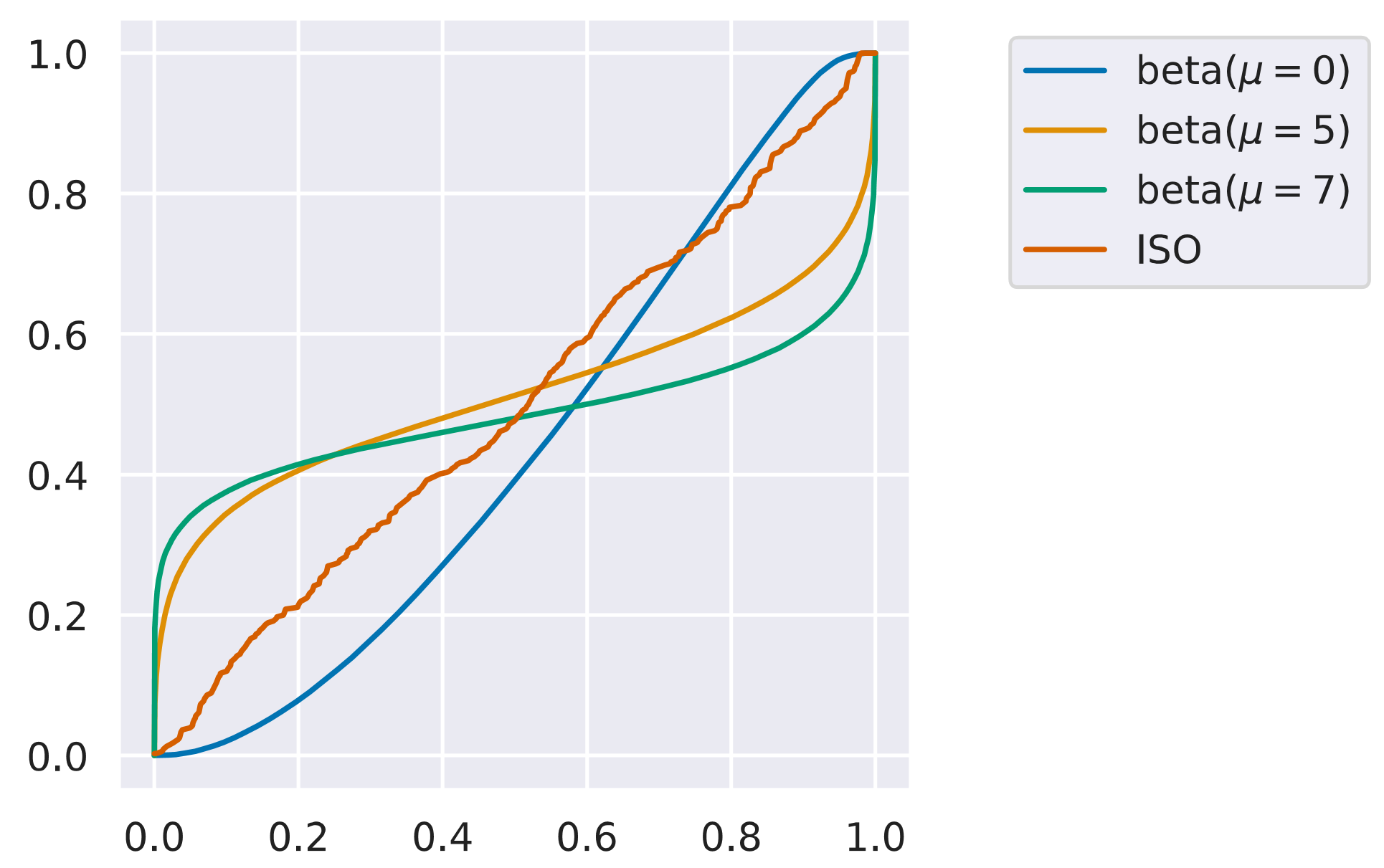}
    \caption{(\textbf{Left Three}) The latent Beta parameters modelled by the GP, with red points representing the inducing points. The red shade shows the area for one standard deviation at each side. Notice here the x-axis represent the mean value predicted by the \ac{OLS}, and hence corresponds to the y-axis in \Cref{fig:toy_res}. (\textbf{Right}) Calibration maps from \GPB{} and isotonic regression. \GPB{} is capable of predicting different calibration maps for different model output, the example shows three calibration maps at $0$, $5$, $7$ respectively. Isotonic regression only gives a single calibration map for all model outputs, and it only aims to calibrate the marginal quantiles.}
    \label{fig:GPB_example}
\end{figure*}

Our contributions are as follows:
\begin{enumerate}[topsep=0pt,itemsep=-1ex,partopsep=1ex,parsep=1ex]
    \item We introduce the concept of distribution calibration, and demonstrate that being calibrated on a distribution-level will naturally lead to calibrated quantiles. 
    %Measures to examine and verify such distribution calibrations are also given and discussed.
    \item We propose a multi-output \ac{GP}-based approach \cite{alvarez2009sparse} to solve the task of \posthoc{} density calibration. This approach models the distribution over calibration parameters, and uses a novel Beta link function to calibrate any given regression outputs. An example is demonstrated in \Cref{fig:GPB_example}.
    \item To ensure the scalability of the model, we further provide a solution based on stochastic variational inference together with induced pseudo-points.
    \item Finally, the proposed approach is experimentally analysed on different regression models including \ac{GP} regression \cite{Rasmussen2006} and \ac{BNN} regression \cite{Gal2016}.
\end{enumerate}

% Paper Structure
The rest of the paper is organised as follows.
In \Cref{sec:preliminaries} we introduce calibration in the context of both classification and regression, together with some related \posthoc{} calibration approaches. 
\Cref{sec:density_level_calibration} defines density-level calibration and discusses some theoretical properties.
The proposed calibration approach is described in \Cref{sec:method}.
Experimental analysis is shown in \Cref{sec:experiments} and \Cref{sec:conclusion} concludes.

\section{Background and Definitions}
\label{sec:preliminaries}

Throughout this paper, $X$ and $Y$ are random variables over spaces $\mathbb{X}$ and $\mathbb{Y}$, where $X$ represents the input features of an instance and $Y$ is the corresponding target value. 
In $k$-class classification $Y$ is categorical with ${\mathbb{Y} = \{1, \dots, k\}}$. % and in regression ${\mathbb{Y} \subseteq \mathbb{R}}$.

\Posthoc{} calibration applies if there is a (pre-trained) probabilistic model, which inputs the feature values and outputs a probability distribution over the target value. 
We use the notation $\mathsf{f}: \mathbb{X} \rightarrow \mathbb{S}_{\mathbb{Y}}$ to denote such a probabilistic model, where ${\mathbb{S}_{\mathbb{Y}}}$ is a space of probability distributions over $\mathbb{Y}$. For the classification task, ${\mathbb{S}_{\mathbb{Y}}}$ consists in vectors $\sb = [s_1, \dots, s_k]$, where $s_i$ denotes the probability of class $i$. Hence, $s_1,\dots,s_k\in[0,1]$ and $\sum_{j=1}^{k} s_i = 1$. 
%\MK{consider using $\hat{\mathbf{p}}$ instead of $\sb$}

%\tom{I like this para but not sure it belongs here}
The idea of \posthoc{} calibration is to learn a transformation, which takes in the probability distribution as output by the model, and transforms it so that the resulting probability distribution would be better calibrated. 
Intuitively, being calibrated means some kind of an agreement between predicted distribution and actual empirical distribution. 
Next, we will see how existing work has instantiated this intuitive notion for classification and regression, and propose a new notion of being distribution-calibrated for regressors.

\subsection{Calibration in Classification}

% \HS{Reduce text, keep definition, introduce and Beta and Isotonic only, example of calibration map.}

%Let us consider (pre-trained) probabilistic classifiers, defined as a class output a probability distribution over $k$ classes. 
%We use the notation $\mathsf{f}: \mathbb{X} \rightarrow \mathbb{S}_{\mathbb{Y}}$ to denote such a probabilistic classification model, where ${\mathbb{S}_{\mathbb{Y}}}$ is the space of all probability vectors $\sb = [s_1, \dots, s_k]$ with $\sum_{j=1}^{k} s_i = 1$.
% We will assume vectors with such bracket notation are column vectors. 
% any row vector will be denoted with a transpose such as $\sb^{\mathsf{T}}$. 
%A calibrated classifier can then be defined as follows.
%A common definition for classification calibration is as follows \cite{Kull2015}:
Classification calibration is defined as follows \cite{Kull2015}:

\begin{definition}[Calibrated Classifier] \label{def:class_cal}
Assume we have a pair of jointly distributed random variables ${(X, Y)}$ over $\mathbb{X}$ and ${\mathbb{Y}=\{1,\dots,k\}}$, and a model ${\mathsf{f}: \mathbb{X} \rightarrow \mathbb{S}_{\mathbb{Y}}}$.
Denoting ${S = \f(X)}$ as the random vector of predicted class probabilities, $\f$ is said to be calibrated iff ${\forall \sb = [s_1,\dots,s_k] \in \mathbb{S}_{\mathbb{Y}}}$, the following holds:
\begin{align}
    \Pr(Y=j \mid S = \sb) = s_j\,.
\end{align}
\end{definition}
Alternative definitions exist, e.g. \cite{Guo2017} require the accuracy on all instances with the same confidence level (highest predicted probability across classes) to agree with the confidence value.

As mentioned above, some models tend to give uncalibrated outputs due to certain computational heuristics or unrealistic assumptions.
Some \posthoc{} approaches are hence proposed to adjust such outputs to yield better calibrated probabilities.
Broadly speaking, most \posthoc{} approaches can be formalised as a calibration map $\mathsf{c}: \mathbb{S}_{\mathbb{Y}} \rightarrow \mathbb{S}_{\mathbb{Y}}$.
In the binary classification case, a calibration map can be seen as a function $\mathsf{c}: [0, 1] \rightarrow [0, 1]$, transforming the probability of class $1$, as the class $2$ is simply the complement. The calibration map can then be visualised in a unit square as in \Cref{fig:GPB_example} (Right). % \tom{think this figure should actually appear here if possible (possibly all toy figures?}
We hence proceed by introducing two illustrated calibration approaches for binary classification.

%\begin{figure}
%    \centering
%    \includegraphics[width=0.2\textwidth]{Figures/beta_example.png}
%    \caption{Calibration maps (Placeholder, will refine details later.)}
%    \label{fig:cal_map}
%\end{figure}

% \subsubsection{Isotonic Regression}
\textit{Isotonic calibration} is a powerful non-parametric method based on %(unweighted linear order) 
isotonic regression %and the PAV algorithm 
along with a simple iterative algorithm called \ac{PAV}, which finds the train-optimal regression line (calibration map) among all non-decreasing functions
%and it is one of the major non-parametric calibration methods 
\cite{Zadrozny2002,Fawcett2007}.
The method calibrates a model by recursively averaging neighbouring non-monotonic scores, so that a piece-wise constant non-decreasing calibration map is obtained at the end.  
%The constraints ensure that the resultant calibrated probabilities are non-decreasing. 
The main issue with isotonic calibration is that it is prone to overfitting on smaller datasets.
\textit{Beta calibration} \cite{Kull2017} is a recently proposed parametric approach for calibration of probabilistic two-class classifiers.
The method has been derived from the assumption that among all instances of any one of the classes, the predicted probability to belong to class $1$ is distributed according to a Beta distribution. 
%that is $\mathsf{p}(S_j = s_j \mid Y = j) = \mathsf{beta}(s_j ; \alpha_j, \beta_j)$, where $\mathsf{beta}$ is the Beta density function. 
Then, the calibration map as a transformation of the predicted probability to belong to class $1$, can be shown to have the following parametric form:
\begin{align}
%    \mathsf{p}(S_j = s_j \mid Y = j) &= \mathsf{b}(s_j ; \alpha_j, \beta_j) \\ \nonumber
    \mathsf{c}_{\beta}(\mathbf{s}) &= \mathsf{\Phi}\Big(a \: \mathsf{ln} \: s_1  - b \: \mathsf{ln} \: s_2 + c\Big)\,,
\end{align}
where %$\mathsf{beta}$ is the Beta density function with parameter $\alpha$ and $\beta$
$\mathsf{\Phi}(z) = \left(1 + e^{-z}\right)^{-1}$ is the logistic sigmoid function, $a$, $b$ and $c$ are parameters depending on the Beta parameters, as well as the marginal class distribution $\mathsf{P}(Y)$.
Beta calibration has the advantage of being naturally defined as a valid calibration map on the interval $[0, 1]$, while providing flexible shapes including sigmoids, inverse-sigmoids, and the identity map (\ie{} no transformation is applied), as shown in \Cref{fig:GPB_example} (right).
Experimentally, it has been shown to provide good calibration results on various binary models such as AdaBoost, \aclp*{SVM}, and \acp{NN} \cite{Kull2017,Kull2017b}.
%Logistic calibration makes an implicit assumption that the per-class scores are Gaussian distributed with equal variance. 
%However, it has been shown that many classifiers (\eg Na{\"i}ve Bayes, AdaBoost) suffer from a particular distortion where these score distributions are heavily skewed (\ie{} non-Gaussian). % \cite{Kull2017}. 
%Here logistic calibration can easily yield probability estimates that are less calibrated than the probabilities output by the classifier originally. Moreover, the logistic curve family does not include the identity function, and hence logistic calibration can even un-calibrate a perfectly calibrated classifier. 
%Beta calibration tackles these problems by using a richer class of calibration maps based on the beta distribution.
%Our proposed method in \Cref{sec:method} takes inspiration from beta calibration, using the same flexible parametric mapping function.

\subsection{Quantile-Calibrated Regression}

%While probabilistic classifiers output a probability distribution over classes, a probabilistic regression model outputs a probability distribution over the real numbers, usually represented by either a \ac{PDF} or \ac{CDF}. 
%For regression we assume $\mathbb{Y} = \mathbb{R}$.
% while other bounded continuous space is also applicable.
%A probabilistic regression model can then be defined as ${\mathsf{f}: \mathbb{X} \rightarrow \mathbb{D}_{\mathbb{Y}}}$,
%where $\mathbb{D}_{\mathbb{Y}}$ is the space of all possible density functions over $\mathbb{Y}$.
In regression, calibration has been traditionally addressed through quantiles \cite{Beran2005,Rueda2007,Taillardat2016,Fasiolo2017,Kuleshov18}. 
%\MK{is it 1 paper or a tradition? Either way, we need a reference here I think}
The goal of a quantile regression model is for a given instance with feature vector $\mathbf{x}$ and for a given quantile $\tau \in [0, 1]$, to find an estimate $\hat{y}$ such that $\mathsf{P}(Y_{\tau} \leq \hat{y} \mid X = \mathbf{x}) = \tau$.
%We hence define a quantile regression model as ${\mathsf{g}: \mathbb{X} \times [0, 1] \rightarrow \mathbb{R}}$.  
The definition of calibrated quantile regression can then be given as:
\begin{definition}[Quantile-Calibrated Regressor]
\label{def:q_cal_reg}
Suppose we have a pair of jointly distributed random variables ${(X, Y)}$ over $\mathbb{X}$ and ${\mathbb{Y} \subseteq \mathbb{R}}$, and a quantile regression model ${\mathsf{g}: \mathbb{X} \times [0, 1] \rightarrow \mathbb{Y}}$. 
Denoting ${G_{\tau} = \mathsf{g}(X, \tau)}$ as the random variable of the $\tau$-quantile predictions, $\mathsf{g}$ is said to be quantile-calibrated iff, ${\forall \tau \in [0, 1]}$, the following holds:
\begin{align}
\label{eq:quantile_calibration}
    \Pr(Y \leq G_{\tau}) = \tau\,.
\end{align}
\end{definition}
%
% \subsection{Approaches for \PostHoc{} Calibration}
Motivated by the above definition, \cite{Kuleshov18} introduces a quantile calibration map, which maps any quantile $\tau\in[0,1]$ into $\mathsf{c}(\tau)=\Pr(Y \leq G_{\tau})\in[0,1]$. After applying this calibration map, the obtained quantile regression model is indeed quantile-calibrated, as $\mathsf{g}_{\mathsf{cal}}(X,\mathsf{c}(\tau))=\mathsf{g}(X,\tau)$ and therefore, $\Pr(Y \leq\mathsf{g}_{\mathsf{cal}}(X,\mathsf{c}(\tau)))=\Pr(Y \leq\mathsf{g}(X,\tau))=\mathsf{c}(\tau)$.
To learn the mapping $\mathsf{c}$, \cite{Kuleshov18} use isotonic regression in the following way. 
Denoting by $\tau_i$ the quantile which corresponds to the actual target value $y_i$ of the training instance $\mathbf{x}_i$, the proposed method starts by collecting the empirical frequency $\bar{\tau_i} = n^{-1}{\sum_{j=1}^{n} \mathsf{I}(\tau_j \leq \tau_i)}$ (where $\mathsf{I}(\cdot)$ is the indicator function), that quantifies the proportion of instances receiving a \ac{CDF} value no greater than each given $\tau_i$.
Isotonic regression can then be applied to learn a mapping using the gathered pairs $(\tau_i, \bar{\tau}_i)_{i=1}^{n}$ to provide better calibrated quantiles according to the collected empirical frequency.

\begin{remark}[Global vs Local Calibration]
Comparing \Cref{def:class_cal} with \Cref{def:q_cal_reg}, we see that classification is defined through a conditional probability, while for quantile regression it is only through a marginal probability.
% As mentioned above, this brings a inconsistency between the two scenarios.
If a calibrated classifier provides a prediction with probability vector $\mathbf{s}$, then on average over all cases with the same prediction, the corresponding targets are distributed according to $\mathbf{s}$.  
In contrast, if a quantile-calibrated regression model provides a prediction with some empirical moments (\eg{} mean and variance), then we cannot claim that on average over all cases with the same predicted moments, the target variable would indeed have a distribution with the same true moments. % mean $\mu$ and standard deviation $\sigma$. 
% Instead, the agreement holds only when averaging over all different predictions.
This is because a marginal probability only considers the problem on a global level, which only guarantees the quantile to be calibrated averaged over all predictions.
This property becomes a disadvantage when the goal is to quantify the uncertainties on each individual prediction.
To mirror more closely the classification case, we next propose a stronger definition of calibration for regression. % in \Cref{sec:density_level_calibration}. 
\end{remark}

\section{Distribution-Calibrated Regression}
\label{sec:density_level_calibration}

The following definition originates from the same principle as \Cref{def:class_cal} for classification, in the sense that the distribution of the target variable $Y$ is required to agree with the output of the model conditioned on the output of the model.
We first choose ${\mathbb{S}_{\mathbb{Y}}}$ to consist of all possible \acp{PDF} over a real-valued target variable. 
Note that by this we are restricting ourselves to working with absolutely continuous distributions. 
%In this section we introduce the concept of distribution-calibrated regression models.
%While probabilistic classifiers output a probability distribution over classes, a probabilistic regression model outputs a probability distribution over the real-valued target variable. 
%We assume that this distribution is absolutely continuous, meaning that it can be represented both as a \ac{PDF} or as a \ac{CDF}.
%For regression we assume $\mathbb{Y} = \mathbb{R}$.
% while other bounded continuous space is also applicable.
%A probabilistic regression model can then be defined as ${\mathsf{f}: \mathbb{X} \rightarrow \mathbb{D}_{\mathbb{Y}}}$,
%where $\mathbb{D}_{\mathbb{Y}}$ is the space of all possible density functions over $\mathbb{Y}$.
%We denote the space of such distributions by $\mathbb{S}_{\mathbb{Y}}$, with $\mathsf{s} \in \mathbb{S}_{\mathbb{Y}}$ being a instance of PDF within the space.
%Next we define when a probabilistic regression model is distribution-calibrated, which is directly analogous to the definition of calibrated classifiers.
%We use the same notation $\mathsf{f}: \mathbb{X} \rightarrow \mathbb{D}_{\mathbb{Y}}$ to denote a probabilistic regression model (also called conditional distribution estimator) which outputs the distribution of the target variable instead of its point-estimate as usual regression models do. , where $\mathbb{D}_{\mathbb{Y}}$ represents the space of all possible density functions on $\mathbb{Y}$.
%
\begin{definition}[Distribution-Calibrated Regressor] \label{def:dens_cal_reg}
Suppose we have a pair of jointly distributed random variables ${(X, Y)}$ over $\mathbb{X}$ and ${\mathbb{Y} \subseteq \mathbb{R}}$, and a model ${\f: \mathbb{X} \rightarrow \mathbb{S}_{\mathbb{Y}}}$. 
Denoting ${S = \f(X)}$ as the random variable of model predictions, $\f$ is said to be distribution-calibrated if and only if ${\forall \mathsf{s} \in \mathbb{S}_{\mathbb{Y}}}$, ${\forall y \in \mathbb{Y}}$, the following equality holds:
\begin{align}
\label{eq:density_calibration}
    \pr(Y=y \mid S = \mathsf{s} ) = \mathsf{s}(y)\,.
\end{align}
\end{definition}
%
%This definition follows the same principle as in classification, in the sense that the distribution of the target variable $Y$ agrees with the output of the model when conditioned on the output of the model.
%Calibration is now defined through a conditional probability density given a particular prediction of $\d$, which follows the same principle as in classification.
In particular, this definition implies that if a calibrated model predicts a distribution with some mean $\mu$ and variance $\sigma^2$, then it means that on average over all cases with the same prediction the mean of the target is $\mu$ and variance is $\sigma^2$.

Next, we show that if the probabilistic regression model is distribution-calibrated then for any $\tau\in[0,1]$ extracting the $\tau$-quantile from the output distribution results in a quantile-calibrated regressor.
\begin{theorem}%[Density calibrated is sufficient for quantile calibrated.]
Let ${\f: \mathbb{X} \rightarrow \mathbb{S}_{\mathbb{Y}}}$ be a distribution-calibrated probabilistic model, and let  
${\mathsf{g}: \mathbb{X} \times [0, 1] \rightarrow \mathbb{Y}}$ be a quantile regressor defined by $\mathsf{g}(\mathbf{x},\tau)=y$ such that 
$\mathsf{P}_{\f(\mathbf{x})}(Y \leq y)=\tau$ where $\mathsf{P}_{\f(\mathbf{x})}$ is the probability measure corresponding to the distribution $\f(\mathbf{x})$.
Then $\mathsf{g}$ is quantile-calibrated.
%Using the notation above, ${\forall \tau \in [0, 1]}$, ${\forall \q \in \mathbb{Q}_{\mathbb{Y}}}$, ${\Pr(Y \leq G_{\tau}) = \tau}$ is satisfied whenever ${\Pr(Y \leq t \mid Q = \q ) = \q(t)}$ is satisfied.
\end{theorem}
\begin{proof}
First let us prove that $\mathsf{g}$ exists. Since $\f(\mathbf{x})$ is absolutely continuous then its \ac{CDF} is continuous. As it is monotonic in the range $[0,1]$ it achieves all values, including $\tau$. Therefore, the required $y$ exists for any $\mathbf{x}$ and $\tau$ and so $\mathsf{g}$ is correctly defined. It remains to prove that ${\Pr(Y \leq \mathsf{g}(X,\tau)) = \tau}$ for any $\tau\in[0,1]$. As $\f$ is distribution-calibrated then for any $\mathsf{s}\in\mathbb{S}_{\mathbb{Y}}$ we get $\Pr(Y=y\mid \f(X)=\mathsf{s})=\mathsf{s}(y)$ for any $y$. Combining this over all $y\leq \mathsf{g}(X,\tau)$ and considering that $\mathsf{g}(X,\tau)$ is uniquely determined by $\f(X)$, we obtain $\Pr(Y\leq \mathsf{g}(X,\tau)\mid \f(X)=\mathsf{s})=\mathsf{P}_{\f(X)}(Y\leq \mathsf{g}(X,\tau))$ which is equal to $\tau$ due to the definition of $\mathsf{g}$. Since $\Pr(Y\leq \mathsf{g}(X,\tau)\mid \f(X)=\mathsf{s})=\tau$ for any $\mathsf{s}$ then ${\Pr\Big(Y\leq \mathsf{g}(X,\tau)\Big)=\tau}$ implying that $\mathsf{g}$ is quantile-calibrated.
% 
% With \Cref{eq:cdf_quantile}, we have:
% %
% \begin{align*}
% \Pr(Y \leq \q^{-1}(\tau) \mid Q=\q) = \tau    
% \end{align*}
% %
% Therefore, we can derive:
% %
% \begin{align*}
%     \Pr(Y \leq G_{\tau}) &= \int_{\mathbb{Q}_{\mathbb{Y}}} \Pr(Y \leq \mathsf{q}_{\tau}^{-1} \mid Q = \q) \: \pr (Q=\q) \: d\q \\
%     &= \tau \int_{\mathbb{Q}_{\mathbb{Y}}} \pr(Q = \q) \: d\q \\
%     &= \tau
% \end{align*}
\end{proof}
On the other hand, since $\mathsf{P}(Y \leq G_\tau) = \tau$ doesn't provide any information about $\mathsf{P} (Y \leq G_\tau\mid \f(X)=\mathsf{s})$, we don't necessarily get a distribution-calibrated model from a quantile calibrated model.

% \HS{I'll put below in common text other than theorem, not that novel.}

While we have provided the definition of distribution calibration and shown its relationship to quantile calibration, we now demonstrate the quantitative benefits for any model being calibrated on a distribution level. 
The metric we adopt is the \ac{NLL}, also known as the log-loss in the context of proper scoring rules \cite{gneiting2007strictly}.

%\begin{theorem}%[Density calibrated is sufficient for quantile calibrated.]
%Assume we have a pair of jointly distributed random variables ${(X, Y)}$ over space $\mathbb{X}$ and ${\mathbb{Y} \subseteq \mathbb{R}}$, and a pre-trained model ${\f: \mathbb{X} \rightarrow \mathbb{S}_{\mathbb{Y}}}$. 
%Then the expected \ac{NLL} $\mathsf{E}_{(\mathbf{X}, Y)}[- \mathsf{ln} \: \mathsf{s}_{\mathbf{x}}(y)]$ is minimised, iff, the model $\mathsf{f}$ is distribution calibrated. 
%\end{theorem}

%\begin{proof}
As shown in \cite{Kull2015}, the expected NLL can be decomposed into the following two terms:
\par\nobreak\vspace{-\baselineskip}
{\small
\begin{align}
    &\mathsf{E}_{(\mathbf{X}, Y)}\Big[- \mathsf{ln} \: \mathsf{s}_{\mathbf{x}}(y)\Big] = \\ \nonumber &\mathsf{E}_{\mathbf{X}}\Big[\mathsf{KL}\Big(\mathsf{p}( Y \mid \mathsf{s}_\mathbf{x}) , \mathsf{s}_{\mathsf{x}}(Y)\Big)\Big] + \mathsf{E}_{(\mathbf{X}, Y)}\Big[- \mathsf{ln} \: \mathsf{p}( y \mid \mathsf{s}_\mathbf{x})\Big] \,,
\end{align}
}%
where
\par\nobreak\vspace{-\baselineskip}
{\small
\begin{align*}
    \mathsf{KL}\Big(\mathsf{p}( Y \mid \mathsf{s}_\mathbf{x}) , \mathsf{s}_{\mathsf{x}}(Y)\Big) = \int_{y} \mathsf{p}( Y \mid \mathsf{s}_\mathbf{x}) \: \mathsf{ln} \frac{\mathsf{p}( Y \mid \mathsf{s}_\mathbf{x})}{\mathsf{s}_{\mathsf{x}}(y)} \mathsf{d} \: y \,.
\end{align*}
}%
%\end{proof}
%
The first term is commonly known as the calibration loss and the latter term is the so called refinement loss.
Once a model is trained, the latter term becomes fixed, as the distribution over $S$ is learnt.
Therefore, the whole expectation can be minimised if and only if the formal $KL$ divergence becomes $0$ everywhere, which indicates calibrated distributions as we defined.
In general, distribution-level calibration ensures we have the most accurate uncertainty from the model predictions of the targets receiving the same prediction.
Such calibration properties allow us to make optimal decisions for each individual prediction. 
%Furthermore, while quantile-level calibration can only address the probability for the target being within a interval, density calibration can further answer questions such as a accurate density ratio between two target values.

%From an evaluation point of view, fixing a trained model, having such calibration properties will yield optimised predictive log-likelihood (hence a $0$ Kullback-Leibler divergence between the predicted and true conditional PDF) as well as \ac{MSE}. 
%Fixing a trained model $\mathsf{f}$, denoting the density predicted by $\mathsf{f}$ as $\mathsf{d}$, we can easily show, given $ \forall \hat{\mathsf{d}}$ :
%
%\begin{align*}
%    \mathsf{KL}\Big( \hat{\mathsf{d}}\: ; \: \mathsf{p}(Y \mid D=\mathsf{d}) \Big) \geq 0
%\end{align*}
%
% And
%
%\begin{align*}
%    \mathsf{MSE}\Big(Y \mid D=\mathsf{d}\: ; \: \mathsf{E}_{\hat{\mathsf{d}}} (Y) \Big) \geq \mathsf{var}_{\mathsf{d}}(Y \mid D=\mathsf{d}) 
%\end{align*}
%
%In both cases, the equity only holds if $\hat{\mathsf{d}} = \mathsf{d}$.
%While here we only consider a single $\mathsf{d}$, notice the marginal loss is the expected loss over $\mathsf{p}(\mathsf{d})$, as for a density calibrated model the calibration property holds for $\forall d \in \mathbb{D}_{\mathbb{Y}}$, thus the marginal loss is also minimised only when we have a calibrated model.
%Here we skip the detailed proof for simplicity, refer to \cite{Kull2015} for a more detailed analysis in the context of Proper Scoring Rules.

\section{Methodology}%The \GPB{} Approach}
\label{sec:method}

%In this section we provide the ingredients to the proposed \GPB{} approach, aiming to provide better calibrated distributions for any given regressor. 
The proposed idea is to use Beta calibration maps to transform CDFs of the distributions output by the regressor, similarly as isotonic maps are used by \cite{Kuleshov18}. 
Unlike \cite{Kuleshov18}, we learn a \ac{GP} to predict the parameters $a,b,c$ of the Beta calibration map from the mean and variance as predicted by the regressor. 
Let us now look into these steps in more detail.

\subsection{Beta link function for regression}
We first adopt the parametric Beta calibration map family \cite{Kull2017, Kull2017b} as a tool to calibrate the \ac{CDF} of any regression output, by transforming quantiles with a beta calibration map $[0,1]\to[0,1]$. 
As the Beta family contains the identity map ($a=1,b=1,c=0$), the regression output can remain the same, if already calibrated. 
Changing $c$ pushes the distribution to the left ($c>0$) or right ($c<0$). Sigmoids ($a,b>1$) decrease the variance of the regression output distribution, while inverse sigmoids ($a,b<1$) increase the variance. 
Changing the balance between $a$ and $b$ makes the distribution skewed to the left ($a<b$) or right ($a>b$).

Beta calibration map applies to the CDFs, while in order to later learn a GP we will need to know the transformation as a link function which directly applies to the PDFs. 
To derive this Beta link function, we need to differentiate the new CDF obtained after applying Beta calibration map $\mathsf{c}_{\beta}$. 
Denoting the quantile by $\mathsf{q}_y$, the differentiation results in the following:
\begin{align}
    \frac{\mathsf{d} \: \mathsf{c}_{\beta}\left(\mathsf{q}_y\right)}{\mathsf{d} \: y} & = \frac{\mathsf{d} \: \mathsf{c}_{\beta}\left(\mathsf{q}_y\right)}{\mathsf{d} \: \mathsf{q}_y}\frac{\mathsf{d} \: \mathsf{q}_y}{\mathsf{d} \: y}
    =  \mathsf{r}_{\beta} \left(\mathsf{q}_y\right) \mathsf{s}_y \,,
\end{align}
where $\mathsf{r}_{\beta} (\mathsf{q})$ is the link function that we were looking for:
\par\nobreak\vspace{-\baselineskip}
{\small
\begin{align}
    \mathsf{r}_{\beta} (\mathsf{q}_y) & = \frac{\mathsf{d} \: \mathsf{\Phi}\Big(a \: \mathsf{ln} \: \mathsf{q}_y  - b \: \mathsf{ln} \: (1-\mathsf{q}_y) + c\Big)}{\mathsf{d} \: \mathsf{q}_y}\\ \nonumber
    & = \frac{\mathsf{q}^{a}_y(1-\mathsf{q}_y)^{b} e^{-c}\left(a - (a-b)\mathsf{q}_y\right)}{\mathsf{q}_y(1-\mathsf{q}_y)\left(\mathsf{q}^{a} + (1-\mathsf{q}_y)^{b} e^{-c}\right)^2} \,.
\end{align}
}%
Here $a$, $b$ and $c$ are the parameters of Beta calibration.
This link function acts as a density ratio between the calibrated PDF and the original PDF as output by the regressor.

%from which we derive a link function to directly evaluate the target density after calibration.

%\subsection{The Beta Link Function}

%As introduced above, under a distribution setting, any monotonic function satisfying $\mathsf{c}: [0,1] \rightarrow [0,1] $ can be used to calibrate a \ac{CDF}.
%We therefore propose to the adopt Beta calibration in the regression setting, given the flexibility demonstrated in the classification case.

%Compared to isotonic regression, Beta calibration has the advantage of being differentiable, leading to an equivalent calibration on the \ac{PDF}:
%
%\begin{align}
%    \frac{\mathsf{d} \: \mathsf{c}_{\mathsf{b}}\left(\mathsf{q}_y\right)}{\mathsf{d} \: y} & = \frac{\mathsf{d} \: \mathsf{c}_{\mathsf{b}}\left(\mathsf{q}_y\right)}{\mathsf{d} \: \mathsf{q}_y}\frac{\mathsf{d} \: \mathsf{q}_y}{\mathsf{d} \: y} \\ \nonumber
%    & =  \mathsf{r}_{\mathsf{b}} \left(\mathsf{q}_y\right) \mathsf{s}_y
%\end{align}
%
%With $\mathsf{r}_{\beta} (\mathsf{q})$ being:
%
%\begin{align}
%    \mathsf{r}_{\mathsf{b}} (\mathsf{q}_y) = \frac{\mathsf{q}^{a}_y(1-\mathsf{q}_y)^{b}e^{-c}\left(a - (a-b)\mathsf{q}_y\right)}{\mathsf{q}_y(1-\mathsf{q}_y)\left(\mathsf{q}^{a} + (1-\mathsf{q}_y)^{b}e{-c}\right)^2}
%\end{align}
%
%Here $a$, $b$ and $c$ are the parameters of Beta calibration.

%Therefore, applying a Beta calibration on the \ac{CDF} can then be seen as multiplying a density ratio $\mathsf{r}_{\beta}$ with the \ac{PDF}.
A sufficient and necessary condition for the ratio to be non-negative is $a \geq 0$ and $b \geq 0$, which is the same condition as required to have a monotonically increasing Beta calibration map $\mathsf{c}_{\beta}$.
Similarly, the parameter setting $a=1$, $b=1$ and $c=0$ gives a constant ratio of $1$ (\eg{} no adjustment on the \ac{PDF}), corresponding to the identify calibration map of $\mathsf{c}_{\beta}$.
Furthermore, as it is defined over the \ac{CDF}, $\mathsf{r}_{\beta}$ has the advantage of always yielding normalised distribution after the multiplication, which can otherwise only be achieved through constrained optimisation for common models in the field of density ratio estimation \cite{Sugiyama2012}. 

\subsection{The \GPB{} Model}

With $\mathsf{r}_{\beta}$ being defined as above, intuitively, the next step to achieve distribution calibration is to construct a model that maps any regression output $(\mu_i, \sigma_i)$ into a set of Beta calibration parameters $(a_i, b_i, c_i)$. 
However, the main challenge of this approach is that, during training, for each $(\mu_i, \sigma_i)$ we normally only observe a single target value $y_i$, which is not enough to learn a good calibration map on the whole distribution.
Therefore, we seek to borrow the observed values from other regression outputs that are close to $(\mu_i, \sigma_i)$, which leads to the choice of the \ac{GP}, a widely adopted non-parametric method within the Bayesian framework.

The proposed model can be formalised in the following way.
We first assume, that there are three latent functions that are jointly distributed with respect to a multi-output \ac{GP} \cite{alvarez2009sparse,skolidis2011,moreno2018}, corresponding to the parameters $a,b,c$ of Beta calibration:
\begin{align}
    (\mathsf{w}_a, \mathsf{w}_b, \mathsf{w}_c) \sim \gp(\mathsf{0}, \mathsf{k}, \B) \,,
\end{align}
where ${\mathsf{k}}$ is the kernel (covariance) function on the regression output distributions, and $\B$ is a $3 \times 3$ coregionalisation matrix modelling the covariance among the outputs.
% $\mathsf{w}_a$, $\mathsf{w}_b$ and $\mathsf{w}_c$.

%\MK{shouldn't we explain the relationship between $\mathsf{w}_a, \mathsf{w}_b, \mathsf{w}_c$ and $a,b,c$ already here?}
%
%Then we construct a \ac{GP} which models the parameters of the Beta calibration as a latent function, while defining a covariance function on all regression outputs through kernel mean embedding \cite{Song2008, Gretton2012, Muandet2017}.
%Finally, we discuss the potential inference approaches for the proposed model, and propose a scalable inference scheme based on pseudo inducing points \cite{Snelson2006,Naish2008,Hensman2015}, (stochastic) variational inference \cite{Nickisch2008,Hoffman2013,Hensman2015,Blei2017}, and path-wise Monte-Carlo gradients \cite{Ruiz2016}.

Regarding the choice of $\mathsf{k}$, here we refer to the well developed area of kernel mean embedding \cite{Muandet2017}.
The idea is to use a kernel function to map each distribution into a \ac{RKHS}, the embedding can then be applied to problems like the two sample test \cite{Gretton2012} and distribution regression \cite{Mitrovic2016}.
Notice that in distribution regression the task is to predict a target variable from a distribution variable, in our model the task is further generalised to infer a distribution over functions from a distribution variable, for the purpose of calibration. 

Calculating the embedding/kernel value generally requires Monte-Carlo samples from the candidate distributions, but can also be analytic under certain combinations of distributions.
Here we choose the univariate Gaussian embedding with \ac{RBF} kernel \cite{Song2008}:
\par\nobreak\vspace{-\baselineskip}
{\small
\begin{align}
    &\mathsf{k}\Big((\mu_1, \sigma_1), (\mu_2, \sigma_2) \Big) = \frac{\theta}{|\sigma_1 + \sigma_2 + \theta^2|^\frac{1}{2}} e^{\left(- \frac{(\mu_1 - \mu_2)^2}{2(\sigma_1 + \sigma_2 + \theta^2)} \right)} \,.
\end{align}
}%
Observe that if we set $\sigma_1=\sigma_2=0$, the kernel reduces to a \ac{RBF} kernel defined over $(\mu_1, \mu_2)$.
%\tom{LHS doesn't cancel?? Should be $\frac{\theta}{\theta}$ no?}

Given $n$ training points $(\bm{\mu}, \bm{\sigma}) = {\Big((\mu_1, \sigma_1), \dots, (\mu_n, \sigma_n)\Big)}$, a Gaussian likelihood on $(w_a^{(i)}, w_b^{(i)}, w_c^{(i)})_{i=1}^{n}$ can then be written as:
\begin{align}
    \pr \left( \begin{bmatrix} \mathbf{w}_a \\ \mathbf{w}_b \\ \mathbf{w}_c \end{bmatrix} \Bigg | \: \bm{\mu}, \bm{\sigma} \right) = \mathsf{N} \left( \begin{bmatrix} \mathbf{w}_a \\ \mathbf{w}_b\\ \mathbf{w}_c \end{bmatrix} \Bigg | \: \zeros, \B \otimes \K \right) \,,
\end{align}
where $\mathsf{N}$ is the likelihood function of multivariate Gaussian, $\K$ is the $n$ by $n$ kernel matrix obtained by applying $\mathsf{k}$ on $(\bm{\mu}, \bm{\sigma})$.
$\mathbf{B}$ is the coregionalisation matrix introduced above, while $\otimes$ denotes the Kronecker product.
$\mathbf{w}_a = [w_a^{(1)}, \dots, w_a^{(n)}]$ and similar for $\mathbf{w}_b$, $\mathbf{w}_c$.

While the form above gives a clear representation of the coregionalisation structure, we use $\C = \K \otimes \B$ for the rest of the paper for convenience, 
\par\nobreak\vspace{-\baselineskip}
{\small
\begin{align*}
    \pr( \mathbf{w} \mid  \bm{\mu}, \bm{\sigma}) &= \mathsf{N} ( \wb  \mid  \zeros, \C )\,, \\
    %\wb &= \begin{bmatrix} \wb_1 \\ \vdots \\ \wb_m \end{bmatrix}, \quad \wb_{i} = [w_a^{(i)}, w_b^{(i)}, w_c^{(i)}]\,.
    \wb &= \begin{bmatrix} \wb_1^\mathsf{T}, \ldots, \wb_m^\mathsf{T} \end{bmatrix}^\mathsf{T}, \quad \wb_{i} = [w_a^{(i)}, w_b^{(i)}, w_c^{(i)}]^\mathsf{T}\,.
\end{align*}
}%
For target values $\mathbf{y} = (y_1, \dots, y_n)$, we can now plug in our Beta link:
\par\nobreak\vspace{-\baselineskip}
{\small
\begin{align*}
\label{eq:marginal_likelihood}
    &\pr(\yb \mid \bm{\mu}, \bm{\sigma}) = \int_{\wb} \Big(\prod_{i=1}^{n}\pr(y_i \mid \mathbf{w}_i, \mu_i, \sigma_i) \Big) \: \pr(\wb \mid \bm{\mu}, \bm{\sigma}) \: d\wb \,, \\
    &\mathsf{p}(y_i \mid \mathbf{w}_i, \mu_i, \sigma_i) = \mathsf{s}_{y_i} \: \mathsf{r}_{\beta}^{(i)}(\mathsf{q}_{y_i}) \,,
\end{align*}
}
where $\mathsf{q}_{y_i}$ and $\mathsf{s}_{y_i}$ represent the Gaussian \ac{PDF} value and \ac{CDF} value at $y_i$ given $\mu_i$ and $\sigma_i$, $\mathsf{r}_{\beta}^{(i)}$ is the link with parameters $a_i$, $b_i$, and $c_i$ given as:
\begin{align*}
    a_i &= e^{(\gamma_a^{-1} \: w_a^{(i)} + \delta_a)} \,, \\
    b_i &= e^{(\gamma_b^{-1} \: w_b^{(i)} + \delta_b)} \,, \\
    c_i &= \gamma_c^{-1} \: w_c^{(i)} + \delta_c \,.
\end{align*}
The exponential function enforces the non-negative constraints on $a$ and $b$.
The hyperparameters $(\gamma_\cdot, \delta_\cdot)$ control the link function at the $\mathbf{0}$-mean \ac{GP} prior.
%To see this, recall that the prior is our belief without seeing any data.
For distribution calibration, a reasonable prior is to use the identity calibration map, indicating that we should not adjust the density functions before seeing any data.
However, while at the prior we have $\left(-e^{\mathsf{E}(w_a^{(i)})} = 1, e^{\mathsf{E}(w_b^{(i)})}=1, \mathsf{E}(w_c^{(i)})=0\right)$ corresponding to the identity map, the Gaussian variance together with the non-linear transform will distort the calibration map after marginalising $\mathbf{w}_i$.
While this distortion cannot be prevented analytically, we use the hyperparameters above to reduce the level of distortion, and optimise them during training time, in the spirit of empirical Bayes methods \cite{robbins1985empirical}.

\subsection{Scalable Inference}

% \HS{Justify VI (factorisation, sparse, intractable link), introduce bound from Hensman 2015. Add a algorithm for MC gradient.}
We have defined $\pr (\yb \mid \bm{\mu}, \bm{\sigma})$ in \Cref{eq:marginal_likelihood}; however the integral is analytically intractable due to the non-linearity in the link function, which makes optimising the hyperparameters challenging.
Furthermore, given a test instance $(\mu_{\star}, \sigma_{\star})$, the calibrated density value at $\mathsf{s}_{\star}(y)$ are also intractable:
\par\nobreak\vspace{-\baselineskip}
{\small
\begin{align*}
    &\mathsf{\hat{s}}_{\star}(y) = %\\ \nonumber
    &\int_{\wb_{\star}} \int_{\wb} \mathsf{r}_{\beta}^{(\star)} \mathsf{s}_{\star}(y) \pr (\wb_{\star} \mid \wb)  \pr (\wb \mid \mathbf{y}, \bm{\mu}, \bm{\sigma}) d\wb  d\wb_{\star} \,,
\end{align*}
}%
%
%which gives two further intractable integrals over $\wb_{\star}$ and $\wb$, as well as an intractable posterior distribution: 
%
\begin{align*}
    % \pr(\wb \mid \mathbf{y}, \bm{\mu}, \bm{\sigma}) = \frac{\Big(\prod_{i=1}^{n}\pr(y_i \mid \mathbf{w}_i, \mu_i, \sigma_i) \Big) \: \pr(\wb \mid \bm{\mu}, \bm{\sigma})}{\pr (\yb \mid \bm{\mu}, \bm{\sigma})} \,,
    \pr(\wb \mid \mathbf{y}, \bm{\mu}, \bm{\sigma}) \propto \Big(\prod_{i=1}^{n}\pr(y_i \mid \mathbf{w}_i, \mu_i, \sigma_i) \Big) \: \pr(\wb \mid \bm{\mu}, \bm{\sigma}) \,,
\end{align*}
Finally, operations on the kernel matrix also present computational challenges as the number of data points grows. 

These issues are analogous to the ones seen in \ac{GP} classification, which also needs to integrate over a Gaussian likelihood to get a non-Gaussian likelihood, as well as dealing with computations on the kernel matrix. 

%To deal with the intractable posterior, common solutions include: sampling \cite{gamerman2006markov}, \ac{EP} \cite{Minka2001}, variational inference \cite{Wainwright2008,Blei2017}, and Laplace approximations \cite{Tierney1986}.
%In this paper we apply the variational inference approach for the following reasons.
%As discussed in \cite{Titsias2008}, while \ac{MCMC} typically provides more accurate solutions, it can be significantly slower than other approaches based on numerical optimisation when the link function can factorise into each data point.
%While working notably well in \ac{GP} classification, \ac{EP} requires matching the two moments to find the site parameters, which is also intractable given our link function.
%Laplace approximation can be attractive for small datasets as it is only requires optimising the un-normalised posterior.
%However, once the dataset becomes large, additional approximations are required, such as the pseudo-points introduced in \cite{Snelson2006}, meaning that finding the MAP solution is no longer straight-forward.

% \tom{Return to this section}

We therefore introduce the scalable inference scheme as proposed in \cite{Hensman2015}, together with the Monte-Carlo gradients approach to address the intractable integration on the link function.
The inference scheme starts with placing a number of $m$ induced pseudo points as in \cite{Snelson2006,Naish2008}, denoting as $(\bm{\mu}_{\mathbf{u}}, \bm{\sigma}_{\mathbf{u}})$, from which we obtain a Gaussian prior $\mathsf{N}(\mathbf{u} \mid \mathbf{0}, \mathbf{C}_{u})$, with $\mathbf{C}_{u}$ being a $3m$ by $3m$ covariance matrix obtained from the kernel function and coregionalisation matrix.
The task is then to approximate a Gaussian posterior $\mathsf{q}(\mathbf{u} \mid \mathbf{m}_\mathbf{u}, \mathbf{V}_\mathbf{u})$ with the parameters $\mathbf{m}_\mathbf{u}$ and $\mathbf{V}_\mathbf{u}$, the \acl{ELBO} as seen in common variational inference approaches:
\begin{align}
    \mathsf{p}(\mathbf{y}) \geq \mathsf{E}_{\mathsf{q}(\mathbf{u})}[\mathsf{ln} \: \mathsf{p}(\mathbf{y}\mid\mathbf{u})] - \mathsf{KL}[\mathsf{q}(\mathbf{u}), \mathsf{N}(\mathbf{u})] \,,
\end{align}
where we omit the dependence on the inputs for simplicity.
While the KL-divergence can be computed analytically, the expectation $\mathsf{E}_{\mathsf{q}(\mathbf{u})}[\mathsf{ln} \: \mathsf{p}(\mathbf{y}\mid\mathbf{u})]$ still remains intractable, as it requires the computation of $\mathsf{ln} \: \int_{\mathbf{w}} \mathsf{p}(\mathbf{y} \mid \mathbf{w})\mathsf{p}(\mathbf{w} \mid \mathbf{u}) \mathsf{d}\:\mathbf{w}$.
As a solution, in \cite{Hensman2015} the authors propose to apply the Jensen inequality again to obtain a further bound:
\par\nobreak\vspace{-\baselineskip}
{\small
\begin{align*}
    \mathsf{E}_{\mathsf{q}(\mathbf{u})}&[\mathsf{ln} \: \mathsf{p}(\mathbf{y}\mid\mathbf{u})] - \mathsf{KL}[\mathsf{q}(\mathbf{u}), \mathsf{N}(\mathbf{u})] \\  & \geq \mathbf{E}_{\mathsf{q}(\mathbf{u})}\Big[\mathsf{E}_{\mathsf{p}(\mathbf{w}\mid\mathbf{u})}[\mathsf{ln}\:\mathsf{p}(\mathbf{y}\mid\mathbf{w})]\Big] - \mathsf{KL}[\mathsf{q}(\mathbf{u}), \mathsf{N}(\mathbf{u})] \,, \\
    & = \mathsf{E}_{\mathsf{q}(\mathbf{w})}[\mathsf{ln}\:\mathsf{p}(\mathbf{y} \mid \mathbf{w})] - \mathsf{KL}[\mathsf{q}(\mathbf{u}), \mathsf{N}(\mathbf{u})] \,,
\end{align*}
}%
where
\par\nobreak\vspace{-\baselineskip}
{\small
\begin{align}
\label{eq:q_w}
    \mathsf{q}(\mathbf{w}) &= \mathsf{N}(\mathbf{w} \mid \mathbf{m}_\mathbf{w}, \mathbf{V}_\mathbf{w}) \,, \\
    \mathbf{m}_\mathbf{w} & = \mathbf{A} \mathbf{m}_\mathbf{u} \,, \nonumber  \\ 
    \mathbf{V}_{\mathbf{w}} &= \mathbf{C} +  \mathbf{A}(\mathbf{V}_{\mathbf{u}}-\mathbf{C}_{\mathbf{u}})\mathbf{A}^{\mathsf{T}} \,, \nonumber \\ 
    \mathbf{A} &= \mathbf{C}_{\mathbf{w}\mathbf{u}}\mathbf{C}_{\mathbf{u}}^{-1} \,, \nonumber
\end{align}
}%
where $\mathbf{C}$, $\mathbf{C}_{\mathbf{u}}$ are as before, $\mathbf{C}_{\mathbf{w}\mathbf{u}}$ is the $3n$ by $3m$ kernel matrix between the training points and the inducing points.

We can now compute the the expectation $\mathsf{E}_{\mathsf{q}(\mathbf{w})}[\mathsf{ln}\:\mathsf{p}(\mathbf{y} \mid \mathbf{w})] $ via Monte-Carlo samples.
In fact, each $\mathbf{ln} \: \mathsf{p}(y_i \mid \mathbf{w}_i)$ can be efficiently computed via three-dimensional Gaussian samples by selecting the corresponding $\mathbf{m}_{\mathbf{w}_i}$ and $\mathbf{V}_{\mathbf{w}_i}$, and performing a reparameterization trick:
\begin{align*}
    \bm{\epsilon}_{\mathbf{w}_i} &= \mathbf{L}_{\mathbf{w}_i} \bm{\epsilon} + \mathbf{m}_{\mathbf{w}_i} \,, \\
    \mathbf{L}_{\mathbf{w}_i} &= \mathsf{Cholesky}(\mathbf{V}_{\mathbf{w}_i}) \,,
\end{align*}
with $\bm{\epsilon}$ being random samples generated from a three dimensional unit Gaussian.
Such a reparameterization allows us to compute the gradient over $\mathbf{m}_{\mathbf{w}_i}$, $\mathbf{V}_{\mathbf{w}_i}$ through the Monte-Carlo integration, which can then be used to optimise all the parameters and hyperparameters.
In general, for the whole model we have the following parameters to optimise: the kernel parameters ($\mathbf{\theta}, \mathbf{B}$), the varational parameters ($\mathbf{m}_\mathbf{u}, \mathbf{V}_\mathbf{u})$, link parameters $(\gamma, \delta)$, and the locations for the inducing points ($\bm{\mu}_{\mathbf{u}}, \bm{\sigma}_{\mathbf{u}})$.
Additionally, as suggested by \cite{Hensman2015}, we replace the parameter $\mathbf{V}_\mathbf{u}$ by its Cholesky factor to ensure $\mathbf{V}_\mathbf{u}$ to be always positive definite.
The overall model can be computed efficiently using modern frameworks supporting automatic differentiation, and can be easily scaled via both online and distributed training using stochastic gradient descent.
To predict a new test instance, the procedure duplicates the one given in \Cref{eq:q_w}: we first compute mean and covariance of $\mathbf{w}_{\star}$ at the test point, then compute the calibrated densities through Monte-Carlo integration. 

%%Edited
The overall computational cost includes the cost of calculating the KL-bound (same as in \cite{Hensman2015}) and the cost of Monte-Carlo integration within gradient calculation. 
During training time, for $m$ inducing points, computing the KL bound requires $\mathcal{O}(m^3)$. 
Computing the link function and its gradient on a batch takes $\mathcal{O}(n*l)$, where $n$ is the size of the mini-batch, and $l$ is the number of Monte-Carlo samples. 
At prediction time, a single instance will cost $\mathcal{O}(m^2+l)$. 
Both $l$ and $m$ can be selected by the user, so the overall time cost is manageable on personal computers and can be significantly shortened using common deep learning frameworks with GPU acceleration.
%%Done

\section{Experiments}
\label{sec:experiments}

% \HS{Add toy example to explain the whole model, and show performance.}

%In this section we provide empirical analysis of the proposed concept of distribution calibration, together with the \GPB{} model.
We now provide emperical analysis of the \GPB{} method in the context of distribution calibration.
\paragraph{Synthetic data.} %We start with a synthetic dataset with the true conditional distribution known.
As shown in \Cref{fig:toy_res}, we generate the synthetic dataset of $360$ points via an equal mixture of two univariate linear models: $y=0.5x + \epsilon $; and $y=\epsilon $, with $\epsilon\sim\mathsf{N}(0, \sqrt{2})$, %. distributed with mean $0$ and std $2$).
uniformly sampled within $[-10, 40]$.
%\MK{perhaps just $y=0.5x$ and $y=0$, with noise distributed as $\mathsf{N}(0,2)$. - but what's the distribution of $x$?}
The applied \ac{OLS} tends to fit a line in the centre as it can only model a uni-modal Gaussian conditional density.%, which is far away from being distribution-calibrated.

We apply both isotonic regression and \GPB{} to examine their behaviour under such a scenario.
The CDF is almost unaffected after applying the isotonic approach, due to the fact that the original \ac{OLS} was close to being (marginally) quantile calibrated. %, caused by the prediction being exactly between the two lines.
Careful inspection shows some fluctuating patterns in the PDFs, caused by the step-wise nature of isotonic regression. Isotonic calibration %makes the calibration map non-continuous and 
leads to non-smooth PDFs after calibrating the quantiles.

Finally, we analyse the results obtained from a \GPB{} model with 32 inducing points.
The model is trained using the ADAM optimiser with a learning rate of $0.01$.
The resulting PDF is able to capture the high density region around the bottom left region, where the two original linear models overlap.
Towards the right side, the \GPB{} model is able to recover the bi-modal nature from the true conditional density, only having access to the Gaussian distribution predicted by the \ac{OLS}. 
On the CDF, we can observe that the \GPB{} model is able to adjust CDF on different scales conditioning on the original model output, which gives better recovery of the true CDF.
This result can be further illustrated through \Cref{fig:GPB_example}. 
%Since in this example we only examine the outputs from \ac{OLS} (e.g. standard deviation is constant), the \GPB{} model can be visualised as common 1-D Gaussian processes.
As the figure indicates, the \GPB{} model provides different estimations for each given output. % $\mu$.
Corresponding calibration maps can hence be obtained via Monte-Carlo samples as mentioned previously, from which we show three examples on the right of \Cref{fig:GPB_example}.
This helps us to further demonstrate the purpose of distribution calibration.
As the isotonic approach only aims for calibrating the quantiles, it uses the same calibration map on each model output, and provides limited improvements for quantifying model uncertainty in the case of our synthetic dataset.
\GPB{}, on the other hand, is designed to work towards calibrated distributions, which by definition requires conditional calibration maps.

\begin{table*}[!t]
    \centering
    \resizebox{0.95\textwidth}{!}{
    \begin{tabular}{c|rrrrrr|rrrrrr|rrrrrr}
    \hline
    \multicolumn{19}{c}{\textbf{OLS}} \\
    \hline
         \small{Dataset}& \multicolumn{6}{c}{\small{NLL}} & \multicolumn{6}{c}{\small{MSE}} & \multicolumn{6}{c}{\small{PBL}} \\
    \hline
    & \tiny{Base} & \tiny{ISO} & \tiny{GPB$_{8}$} & \tiny{GPB$_{16}$}  & \tiny{GPB$_{32}$} & \tiny{GPB$_{64}$} & \tiny{Base} & \tiny{ISO} & \tiny{GPB$_{8}$} & \tiny{GPB$_{16}$}  & \tiny{GPB$_{32}$} & \tiny{GPB$_{64}$} & \tiny{Base} & \tiny{ISO} & \tiny{GPB$_{8}$} & \tiny{GPB$_{16}$}  & \tiny{GPB$_{32}$} & \tiny{GPB$_{64}$} \\
    \hline
1 & 5.37  & 5.78  & \textbf{5.34 } & 5.46  & 5.39  & 5.38  & 2664.69  & 2664.86  & \textbf{2590.12 } & 2603.50  & 2616.56  & 2640.62  & 1720.24  & 1720.84  & \textbf{1684.4 } & 1760.26  & 1715.  & 1710.14 \\
2 & 3.04  & 3.2  & 2.84  & 2.85  & 2.85  & \textbf{2.83 } & 25.28  & 25.31  & 22.03  & 21.20  & 22.23  & \textbf{20.75 } & 178.34  & 176.65  & 160.37  & 157.32  & 159.79  & \textbf{155.17 }\\
3 & 2.99  & 3.14  & 2.93  & \textbf{2.92 } & 2.92  & 2.92  & 23.00  & 22.98  & 21.78  & 21.45  & 21.37  & \textbf{21.28 } & 524.63  & 523.99  & 506.  & 503.39  & 502.78  & \textbf{501. }\\
4 & 1.94  & 2.22  & 1.84  & \textbf{1.81 } & 3.29  & 2.25  & 2.56  & 2.57  & 2.85  & \textbf{2.45 } & 3.2  & 3.83  & 59.94  & \textbf{59.28 } & 64.40  & 60.25  & 69.12  & 76.41 \\
5 & 3.76  & 4.27  & 3.76  & 3.73  & 3.73  & \textbf{3.71 } & 108.77  & 108.55  & 108.89  & 107.84  & 106.08  & \textbf{105.24 } & 788.66  & 789.96  & 789.37  & 779.02  & 771.04  & \textbf{767.79 }\\
6 & 5.94  & \textbf{5.37 } & 5.56  & 5.47  & 5.44  & 5.47  & \textbf{8502.15 } & 8502.19  & 8624.57  & 8528.08  & 8556.49  & 8528.08  & 112557.73  & \textbf{98345.03 } & 100867.49  & 100962.09  & 100708.98  & 100962.09 \\
    \hline
    \multicolumn{19}{c}{\textbf{BR}} \\
    \hline
         \small{Dataset}& \multicolumn{6}{c}{\small{NLL}} & \multicolumn{6}{c}{\small{MSE}} & \multicolumn{6}{c}{\small{PBL}} \\
    \hline
    & \tiny{Base} & \tiny{ISO} & \tiny{GPB$_{8}$} & \tiny{GPB$_{16}$}  & \tiny{GPB$_{32}$} & \tiny{GPB$_{64}$} & \tiny{Base} & \tiny{ISO} & \tiny{GPB$_{8}$} & \tiny{GPB$_{16}$}  & \tiny{GPB$_{32}$} & \tiny{GPB$_{64}$} & \tiny{Base} & \tiny{ISO} & \tiny{GPB$_{8}$} & \tiny{GPB$_{16}$}  & \tiny{GPB$_{32}$} & \tiny{GPB$_{64}$} \\
    \hline
1 & 5.33  & 5.83  & 5.51  & 5.55  & 5.5  & \textbf{5.28 } & 2398.89  & 2404.58  & 2421.27  & \textbf{2333.93 } & 2405.24  & 2421.43  & 1609.63  & 1620.15  & 1738.29  & 1758.86  & 1728.5  & \textbf{1596.14 }\\
2 & 2.94  & 3.06  & 2.74  & 2.73  & 2.73  & \textbf{2.73 } & 20.65  & 20.53  & \textbf{15.45 } & 15.96  & 16.12  & 15.97  & 169.16  & 164.01  & \textbf{142.45 } & 143.54  & 142.63  & 142.58 \\
3 & 2.96  & 3.15  & 2.91  & 2.92  & 2.89  & \textbf{2.89 } & 21.65  & 21.63  & 19.54  & 19.62  & \textbf{19.44 } & 19.46  & 513.84  & 512.89  & 487.04  & 488.23  & 483.71  & \textbf{482.61 }\\
4 & 1.83  & 2.31  & 2.16  & \textbf{1.80 } & 1.85  & 1.81  & \textbf{2.22 } & 2.22  & 2.75  & 2.57  & 2.65  & 2.57  & 55.90  & \textbf{53.22 } & 59.86  & 57.74  & 58.88  & 57.98 \\
5 & 3.76  & 4.01  & 3.75  & 3.73  & \textbf{3.73 } & 3.74  & 107.80  & 107.82  & 108.41  & \textbf{107.78 } & 107.82  & 108.44  & 791.3  & 791.23  & 786.36  & 782.1  & \textbf{779.74 } & 784.40 \\
6 & 5.96  & \textbf{5.4 } & 5.54  & 5.49  & 5.41  & 5.44  & 8750.67  & \textbf{8750.46 } & 9162.78  & 9382.19  & 8800.73  & 8771.09  & 113298.92  & \textbf{98739.99 } & 101289.49  & 106380.51  & 98973.39  & 99643.87 \\
    \hline
\multicolumn{19}{c}{\textbf{NN}} \\
    \hline
         \small{Dataset}& \multicolumn{6}{c}{\small{NLL}} & \multicolumn{6}{c}{\small{MSE}} & \multicolumn{6}{c}{\small{PBL}} \\
    \hline
    & \tiny{Base} & \tiny{ISO} & \tiny{GPB$_{8}$} & \tiny{GPB$_{16}$}  & \tiny{GPB$_{32}$} & \tiny{GPB$_{64}$} & \tiny{Base} & \tiny{ISO} & \tiny{GPB$_{8}$} & \tiny{GPB$_{16}$}  & \tiny{GPB$_{32}$} & \tiny{GPB$_{64}$} & \tiny{Base} & \tiny{ISO} & \tiny{GPB$_{8}$} & \tiny{GPB$_{16}$}  & \tiny{GPB$_{32}$} & \tiny{GPB$_{64}$} \\
    \hline
1 & 7.30  & 6.08  & 5.43  & 5.44  & \textbf{5.42 } & 5.44  & 3584.02  & 3550.86  & 3521.95  & 3542.58  & 3523.05  & \textbf{3517.75 } & 2303.88  & 2004.76  & 1969.62  & 1978.68  & \textbf{1966. } & 1975.43 \\
2 & 2.78  & 2.86  & 2.74  & 2.75  & 2.72  & \textbf{2.72 } & 15.85  & 15.89  & 15.16  & 14.66  & \textbf{14.51 } & 14.69  & 145.90  & 144.40  & 141.61  & 141.01  & \textbf{138.32 } & 139.1 \\
3 & 5.96  & 5.04  & 3.78  & 3.72  & 3.77  & \textbf{3.67 } & 1031.88  & 1367.93  & 64.48  & \textbf{61.53 } & 63.14  & 62.80  & 4044.07  & 3867.07  & 1018.82  & 976.44  & 1003.24  & \textbf{945.84 }\\
4 & 7.49  & 16.71  & 1.68  & 1.68  & 1.64  & \textbf{1.58 } & 2.36  & 2.35  & \textbf{2.3 } & 2.33  & 2.38  & 2.35  & 68.86  & 57.64  & 56.28  & 56.52  & 56.47  & \textbf{56.14 }\\
5 & 3.29  & 3.45  & 3.16  & 3.16  & \textbf{3.16 } & 3.16  & 53.59  & \textbf{41.38 } & 42.69  & 42.57  & 42.90  & 42.79  & 529.13  & \textbf{461.93 } & 464.56  & 465.01  & 466.24  & 465.81 \\
6 & 18.09  & 7.10  & 5.49  & 5.25  & \textbf{5.22 } & 5.25  & \textbf{9833.01 } & 9840.09  & 11206.57  & 12011.97  & 9890.18  & 12011.97  & 110156.3  & \textbf{98102.22 } & 117965.81  & 112022.83  & 98574.75  & 112022.83 \\
    \hline
\multicolumn{19}{c}{\textbf{GP}} \\
    \hline
         \small{Dataset}& \multicolumn{6}{c}{\small{NLL}} & \multicolumn{6}{c}{\small{MSE}} & \multicolumn{6}{c}{\small{PBL}} \\
    \hline
    & \tiny{Base} & \tiny{ISO} & \tiny{GPB$_{8}$} & \tiny{GPB$_{16}$}  & \tiny{GPB$_{32}$} & \tiny{GPB$_{64}$} & \tiny{Base} & \tiny{ISO} & \tiny{GPB$_{8}$} & \tiny{GPB$_{16}$}  & \tiny{GPB$_{32}$} & \tiny{GPB$_{64}$} & \tiny{Base} & \tiny{ISO} & \tiny{GPB$_{8}$} & \tiny{GPB$_{16}$}  & \tiny{GPB$_{32}$} & \tiny{GPB$_{64}$} \\
    \hline
1 & 5.43  & 5.74  & \textbf{5.43 } & 5.43  & 5.43  & 5.43  & 3022.28  & 3022.22  & \textbf{3017.1 } & 3030.25  & 3027.80  & 3017.9  & 1820.38  & 1822.93  & \textbf{1818.42 } & 1821.11  & 1821.67  & 1819.82 \\
2 & 2.59  & 2.77  & 2.46  & 2.43  & 2.44  & \textbf{2.43 } & 9.86  & 9.85  & 9.08  & 8.66  & 8.86  & \textbf{8.51 } & 116.33  & 113.23  & 108.35  & 105.36  & 106.25  & \textbf{104.47 }\\
3 & 3.13  & 3.33  & 3.28  & 3.19  & 3.25  & \textbf{3.12 } & 30.13  & 30.08  & 29.79  & \textbf{29.54 } & 29.55  & 29.54  & 407.47  & 407.28  & 430.65  & 415.09  & 425.39  & \textbf{403.90 }\\
4 & 1.84  & \textbf{-0.41 } & 2.22  & 2.24  & 2.02  & 1.63  & \textbf{2.27 } & 2.27  & 2.47  & 3.28  & 2.4  & 2.3  & 56.47  & 122.1  & 70.53  & 75.39  & 63.13  & \textbf{54.51 }\\
5 & 3.08  & 3.31  & 3.08  & 3.09  & \textbf{3.08 } & 3.08  & 29.91  & 29.91  & 30.24  & 30.55  & 30.11  & \textbf{29.91 } & \textbf{394.32 } & 399.14  & 396.3  & 399.37  & 395.83  & 395.33 \\
6 & 6.28  & 6.08  & 5.63  & 5.75  & 5.80  & \textbf{5.46 } & 13307.05  & 13316.46  & 13548.90  & 13666.95  & 13739.56  & \textbf{13284.93 } & 6343.86  & 5835.73  & 5788.3  & 5818.16  & 5789.79  & \textbf{5762.03 }\\
    \hline
    \end{tabular}}
    \caption{Predictive performance with the different base models, the lowest loss is shown in bold.}
    \label{tab:RES_ALL}
\end{table*}

%While we have demonstrated the idea of distribution calibration and the behaviour of the \GPB{} model, we move to a more quantitative analysis on real world datasets. 
\paragraph{Real world datasets.}
We focus on three evaluation measures: 
\begin{enumerate*}[label=(\roman*)]
    \item predictive \acf{NLL}, 
    \item \ac{MSE}, and
    \item \ac{PBL}. 
\end{enumerate*}
As discussed previously, for a pre-trained model, the NLL will be minimised if a model achieves density-level calibration. 
MSE, on the other hand, is a generic measure to evaluate a model's predictive performance.
Pinball loss commonly used to train and evaluate the calibration of quantiles \cite{Fasiolo2017}, and is defined as follows:
\par\nobreak\vspace{-\baselineskip}
{\small
\begin{align*}
    \mathsf{PBL}(\tau) &= \mathsf{E}_{(\mathbf{X}, Y)}[\mathsf{L}(y, g(\mathbf{x}, \tau))] \,, \\
    \mathsf{L}\Big(y, g(\mathbf{x}, \tau)\Big)& = 
    \begin{cases}
    (1-\tau)\Big(g(\mathbf{x}, \tau) - y\Big) \: \text{if} \: y < g(\mathbf{x}, \tau)\,, \\
    \tau \Big(y - g(\mathbf{x}, \tau)\Big) \quad \quad \quad \text{otherwise}\,.
    \end{cases}
\end{align*}
}%
%
%%Edited
Pinball loss is an asymmetric loss where the overestimation loss and underestimation loss are weighted with the predicted quantiles, and hence specified for each given quantile. 
In the following experiments we calculate the averaged loss from the quantiles of $0.05$ to $0.95$, in increments of $0.05$.
%Done

We select the following four regression models: 
\begin{enumerate*}
    \item \ac{OLS} regression,
    \item \ac{BRR},
    \item \ac{GPR}, and 
    \item \acp{BNN}.% (\acp{NN} with dropout approximations \cite{Gal2016}).
\end{enumerate*}

The first two models provide uniform variance estimates for each instance; \ac{BRR} optimises the variance using an inverse-gamma prior.
The later two provide variance estimates for each instance. 
While \ac{GPR} is derived within the non-parametric Bayesian framework, the \ac{NN} model doesn't provide uncertainty estimations by default, but through the use of dropout approximations, we can obtain some form of uncertainty around the observations \cite{Gal2016}. 

%\MK{Guo not relevant here, as they do classification?}

The experiments are applied on the following UCI datasets (sizes in parentheses):
\begin{enumerate*}%label=\itshape{\arabic*\upshape})]
    \item Diabetes (442), 
    \item Boston (506), 
    \item Airfoil (1503), 
    \item Forest Fire (517), 
    \item Strength (1030), 
    \item Energy (19735). 
\end{enumerate*}
The exception is \ac{GPR}, where we limit the maximal dataset size to $1000$ data points, due to the computational complexity of \acp{GP}.
All the experiments use a random $(0.75, 0.25)$ train-test split, with both the base model and calibrators trained on the same set. 
% Following \cite{Kuleshov18}, we omit a distinct calibration set and use the same training set to optimise the calibration models, since both calibration methods in the experiments are non-parametric or semi-parametric and can potentially benefit from large training sets \cite{Kull2017b},  
During prediction time, $4096$ points with equal distance are selected from ${\mu_{min} - 8 \: \sigma_{max}}$ to ${\mu_{max} + 8 \: \sigma_{max}}$, where $\mu_{min}$, $\mu_{max}$ are the minimal and maximal predicted mean values in the training set, and $\sigma_{max}$ is the maximal of the predicted standard deviation.
This ensures that we cover nearly the whole predicted and calibrated densities, allowing the expected value to be approximated through the trapezoid rule.    

For the \acp{NN}, we use the same setting as in \cite{Kuleshov18}, which is a 2-layer fully-connected structure with 128 hidden units per layer and ReLU activation.
The dropout rate is set to $0.5$, default weight decay of $10^{-4}$ and the length scale of $1.0$ are used to approximate the mean and variance following the results given in \cite{Gal2016}.

We run \GPB{} with $8$, $16$, $32$ and $64$ inducing points, batch size of $128$, and $64$ Monte-Carlo samples per batch to compute the objective function and the gradient.
The parameters are again optimised using ADAM with a learning rate of $0.001$.
The results are given in \Cref{tab:RES_ALL}, showing that for most cases the \GPB{} model is capable of improving the results on all the three evaluation measures.
\GPB{} didn't show improvements in some cases for dataset 3, which has about $20000$ instances. 
Increasing the number of inducing points can be potentially beneficial in such cases according existing sparse GP literature.
Another observation is on dataset 4 with GP as the base model, where the isotonic approach gives a significantly low NLL.
This dataset has a target distribution where about half of the target values are $0$, there is a chance that the isotonic regression happens to assign a step change exactly or very close to $0$, which results in a very high log-likelihood.
%Although the performance is better here, such result is highly stochastic, and cause the isotonic approach to performs worse on other cases.

In summary, the \GPB{} method is clearly superior in the synthetic example where local calibration is required, and in most real-world examples. We also note that for \ac{NLL}, which we argue is the most faithful metric for distribution calibration, \GPB{} almost always performs best irrespective of the choice of the number of inducing points.

\section{Conclusions}
\label{sec:conclusion}

While both calibration of classifiers and quantile regressors have been studied broadly, we introduce the idea of distribution calibration.
Models that are well calibrated on a distribution level provide improved uncertainty quantification on the target variable, as well being calibrated on the quantile level. 

Although %by definition 
distribution calibration is applicable to any conditional density estimator, we focus on a regression setting given its popularity in predictive machine learning tasks.
We propose the \GPB{} approach which combines multi-output \acp{GP} with Beta calibration from binary classification, and distribution regression.
The sparse variational inference scheme allows the model to scale to large datasets and shows strong empirical performance.
%
%Here we would like to cover two major interesting directions to further develop the area of distribution calibration.
%While the Beta link function has certain properties, being parametric means it can only approximate a fixed domain of calibration maps.
%Enlarge the choices of the link function for different scenarios is hence of our primary interests.
%Also, as mentioned, kernel mean embedding can work with any distributions with corresponding samples, the proposed model can hence naturally generalise to other density estimators.
%Investigating related performances would thus become a valuable addition to the future works.
Directions for further work include non-parametric calibration maps, and generalising the model to other forms of density estimation.

%we focus on a \posthoc{} approach to improve predicted uncertainties on regression models, which can be applied to the outputs from any trained models. We first discussed calibration in the context of regression and introduced the concept of distribution calibration, then presented a calibration approach based on multi-output \acp{GP} with a novel link function. We demonstrated empirically that for a set of common regression models this method showed improvements in both predictive performance and calibration level relative to the baselines.
%Interesting directions for future work include calibration for other regression settings such as generalised linear models or multi-output regression.

\clearpage
\newpage

\section*{Acknowledgements}

This work was supported by the SPHERE Interdisciplinary Research Collaboration, funded by the UK Engineering and Physical Sciences Research Council under grant EP/K031910/1.
MK was supported by the Estonian Research Council under grant PUT1458.

\bibliographystyle{icml2019}
\bibliography{main}

\end{document}